\documentclass[letterpaper, 10 pt, conference]{ieeeconf}  

\IEEEoverridecommandlockouts                              

\usepackage{graphics} 
\usepackage{epsfig} 
\usepackage{graphicx}
\usepackage{epstopdf} 
\usepackage{mathptmx} 
\usepackage{times} 
\usepackage{amsmath}
\usepackage{amssymb}
\usepackage{microtype}
\usepackage{dsfont}

\usepackage{algorithm}
\usepackage{algpseudocode}
\usepackage{subcaption}
\usepackage{booktabs} 
\usepackage{multirow}
\usepackage{hyperref}
\hypersetup{breaklinks=true,colorlinks=true,linkcolor=blue,citecolor=blue}
\usepackage{color}

\newcommand{\vect}[1]{\mathbf{#1}}

\DeclareMathOperator{\E}{\mathbb{E}}
\DeclareMathOperator{\Prob}{\mathbb{P}}
\DeclareMathOperator{\Var}{\text{Var}}
\DeclareMathOperator{\Cov}{\text{Cov}}

\newtheorem{remark}{\textbf{Remark}}
\newtheorem{proposition}{\textbf{Proposition}}

\newtheorem{definition}{\textbf{Definition}}

\newtheorem{assumption}{\textbf{Assumption}}

\newtheorem{lemma}{Lemma}
\newtheorem{corollary}{\textbf{Corollary}}

\title{\LARGE \bf
Sampling Complexity of Path Integral Methods for Trajectory Optimization}

\author{Hyung-Jin Yoon$^{1}$, Chuyuan Tao$^{2}$, Hunmin Kim$^{2}$, Naira Hovakimyan$^{2}$, and Petros Voulgaris$^{1}$
\thanks{*Research supported by NSF CPS \#1932529, NSF CMMI \#1663460, UNR internal funding, and UIUC STII-21-06.}
\thanks{$^{1}$Hyung-Jin Yoon and Petros Voulgaris are with the Department of Mechanical Engineering, University of Nevada, Reno, NV 89557, USA
        {\tt\small \{hyungjiny, pvoulgaris\}@unr.edu}}%
\thanks{$^{2}$Chuyuan Tao, Hunmin Kim, and Naira Hovakimyan are with the the Department of
Mechanical Science and Engineering, University of Illinois at Urbana-Champaign, Urbana, IL 61801, USA
        {\tt\small \{hunmin, chuyuan2, nhovakim\}@illinois.edu}}%
}

\begin{document}

\maketitle
\thispagestyle{empty}
\pagestyle{empty}

\begin{abstract}
The use of random sampling in decision-making and control has become popular with the ease of access to graphic processing units that can generate and calculate multiple random trajectories for real-time robotic applications. In contrast to sequential optimization, the sampling-based method can take advantage of parallel computing to maintain constant control loop frequencies. Inspired by its wide applicability in robotic applications, we calculate a sampling complexity result applicable to general nonlinear systems considered in the path integral method, which is a sampling-based method. The result determines the required number of samples to satisfy the given error bounds of the estimated control signal from the optimal value with the predefined risk probability. The sampling complexity result shows that the variance of the estimated control value is upper-bounded in terms of the expectation of the cost. Then we apply the result to a linear time-varying dynamical system with quadratic cost and an indicator function cost to avoid constraint sets.
\end{abstract}

\section{Introduction}
 While making important decisions, it is common to simulate potential future consequences with arbitrary decisions before taking any actions. For critical decision-making, we use known models for simulating scenarios (trajectories) and consider all possible actions (samples from decision set). The aforementioned decision-making procedure that seemingly aligns with our common sense has been implemented in robotic applications. To name a few, rapidly exploring random trees (RRT)~\cite{lavalle1998rapidly} uses random samples to find collision-free paths connecting an initial position to a goal position. Randomly generated trajectories out of a set of motion primitives were used for quadcopter doing impressive maneuvers, including returning a moving ball with a tennis racket mounted on a quadcopter~\cite{mueller2015computationally}. More recently, a sampling-based trajectory optimization~\cite{kappen2005linear} was implemented as a model predictive control to steer a ground vehicle~\cite{williams2016aggressive}.

Compared to the gradient-based optimization applied to trajectory optimization, the sampling-based methods~\cite{lavalle1998rapidly, mueller2015computationally, williams2016aggressive} do not need gradients that might not exist for the optimization problem with discontinuous components. Furthermore, in the latter two sampling-based trajectory optimization applications~\cite{mueller2015computationally, williams2016aggressive}, each sample trajectory can be calculated independently. So, the sampling-based optimization can use thousands of random trajectories generated in parallel, taking advantage of the graphic processing unit (GPU). Since the number of sample trajectories is fixed, in contrast to the uncertain number of iterations with the gradient-based optimization, the sampling-based method would be suitable for applications where fixed computation time for each control loop iteration is required. 

In this paper, we focus on the path integral method~\cite{kappen2005path} that provides a mathematically sound methodology for determining optimal control based on the stochastic sampling of the trajectories. Like other sampling-based methods, the path integral (PI) method does not require gradients of the dynamics or cost functions. This provides flexibility in the form of model or the cost information. For example, the convolutional neural network was used for the cost functions in~\cite{drews2017aggressive}, and nonlinear basis functions were used to model the dynamics~\cite{williams2016aggressive}. Due to the aforementioned advantages, the PI methods find more applications. The PI methods were applied to multiagent systems on cooperative missions~\cite{wan2021cooperative}. The PI methods were extended to incorporate ensemble models approaches~\cite{abraham2020model}. Also, the PI was integrated into navigation for an unmanned-aerial-vehicle that explores while learning the cost functions~\cite{mohamed2020model}.   

In practice, the PI methods typically determine the optimal control using Monte-Carlo (MC) integration~\cite{mackay1998introduction}, i.e., sample-based approximation. Despite the increasing application of the PI methods, there are not enough studies on sampling complexity analysis for PI methods, to the best of the authors' knowledge. This paper studies the required number of samples for discrete-time dynamic systems and shows how the required number of samples depends on the nominal control, cost, and dynamics. We also provide illustrative examples that show how inner loop stabilization that is analogous to using feedback control to steer covariance~\cite{yin2021improving} and constraining the state values, can be useful for reducing sampling complexity.

\subsection{Relevant Research and Contribution}
There are existing papers~\cite{lorenzen2016constraint, mammarella2020computationally, alamo2009randomized} on sampling complexities for stochastic model predictive control. The aforementioned papers~\cite{lorenzen2016constraint, mammarella2020computationally, alamo2009randomized} focus on the certainties of satisfying constraints while solving optimal control problems, e.g., linearly constrained quadratic programming. The certainty of the constraint satisfaction can be calculated using a concentration inequality, e.g., Chernoff bound~\cite{chernoff1952measure} was used in~\cite{alamo2009randomized}. The bounds depend on the complexity of the constraint function using \emph{Vapnik–Chervonenkis dimension}~\cite{vapnik2015uniform}. In contrast, the applications of the PI methods~\cite{williams2016aggressive, wan2021cooperative, mohamed2020model} consider the constraints within the cost functions as penalties.
The PI methods do not solve a linear/quadratic programming for optimal control as in~\cite{lorenzen2016constraint, mammarella2020computationally, alamo2009randomized}, but generates sample trajectories and gets weight-averaged control out of the random control samples where the weight depends on the cost of trajectories, in the form of MC integration. This paper uses concentration inequalities, i.e., Chebyshev inequality~\cite{markov1884certain} and Hoeffding's inequality~\cite{hoeffding1994probability}, in the MC integration that estimates the optimal control in the PI methods.

We obtain a sampling complexity bound given the expectation of the random trajectory cost that is applicable to nonlinear systems, providing guidance to the required number of samples for desired approximation errors. Furthermore, we apply the bound to discrete-time linear systems with  cost functions that consider quadratic state-cost function and indicator function of the set to avoid, e.g., an obstacle that is a convex hull. With this specific system and the cost function, we show the sampling complexity's dependence on the system matrix's stability. We also provide illustrative examples that suggest the advantages of using inner-loop stabilizing control when applying the PI methods.

\section{Preliminaries}
\subsection{Notations}
We use $t$ to denote both the discrete time index and the continuous time variable. The superscript $(n)$ denotes the sample index, e.g., $X^{(n)}$ denotes the $n$th sample of the random variable $X$.
For a matrix $\mathbf{A}$, $\mathbf{A}^\top$ and  $\mathbf{A}^{-1}$ denote the transpose and the inverse of $\mathbf{A}$, respectively. And $\text{diagonal}(\overline{\vect{0}})$ denotes diagonal matrix with zero vector $\overline{\vect{0}}$. For a symmetric matrix $\mathbf{S}$, $\mathbf{S}>0$ and $\mathbf{S}\geq0$ indicate that $\mathbf{S}$ is positive definite and positive semi-definite, respectively. The matrix $\mathbf{I}$ denotes the identity matrix with an appropriate dimension. We use $\Vert \cdot \Vert$ to denote the standard Euclidean norm for a vector or an induced matrix norm if it is not specified, $\E[\cdot]$ to denote the expectation, and $\text{Var}(\cdot)$ to denote the variance. For a vector $\vect{a}$,  $[\vect{a}]_i$ denotes the $i$th element in the vector $\vect{a}$. For a matrix $\vect{A}$,  $[\vect{A}]_{i,j}$ denotes the element at the $i$th row and the $j$th column in the matrix  $\vect{A}$. We denote the minimum singular value and the maximum singular value of the matrix $\vect{A}$ as $\sigma_\text{min}(\vect{A})$ and $\sigma_\text{max}(\vect{A})$ respectively. Also, We denote the minimum eigenvalue value and the $i$th eigenvalue of matrix $\vect{A}$ as $\lambda_\text{min}(\vect{A})$ and $\lambda_\text{i}(\vect{A})$ respectively. And we denote the product of matrices as $\Pi_{i=1}^n \vect{A}_i = \vect{A}_1\vect{A}_2\cdots\vect{A}_n$.

\subsection{Path Integral Methods}
\begin{figure}[thpb]
\centering
 \includegraphics[width=0.35\textwidth]{ 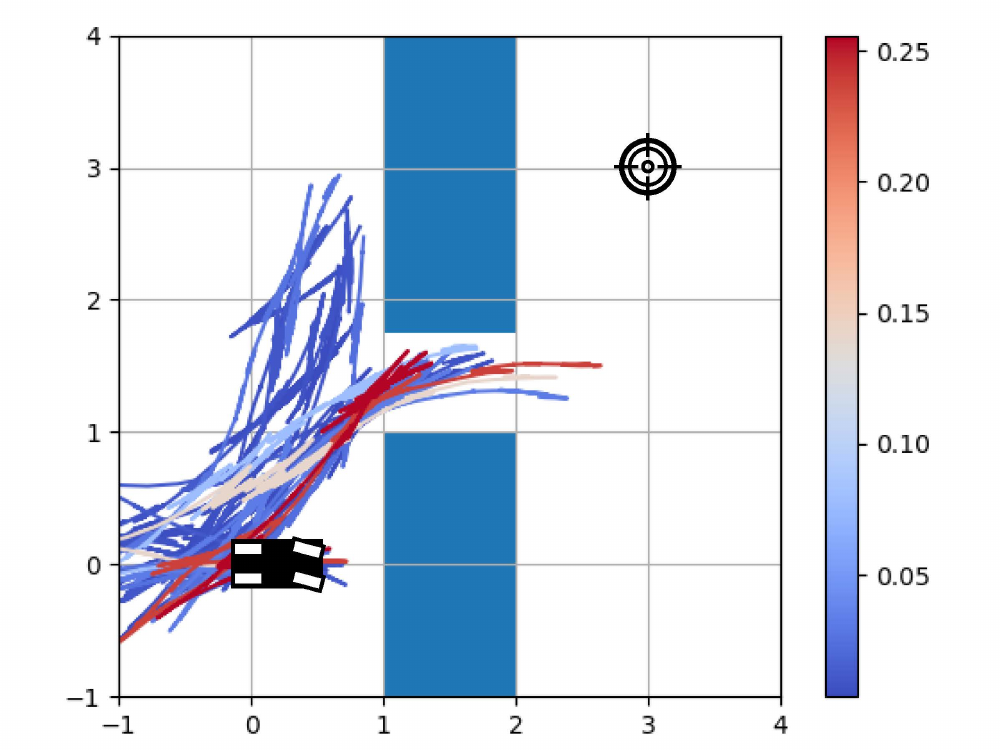}
 \caption{The first ten trajectories out of 10000 samples with the greatest weights (red: greater weights, blue: lower weights)}
\medskip
\label{fig:random_trajectories}
\vspace{-0.35in}
\end{figure}
The PI methods~\cite{williams2016aggressive,kappen2005path} consider the following continuous time stochastic differential equation that is control affine as
\begin{equation*}
    d\vect{x}(t) = \left(f(\vect{x}(t), t) + g(\vect{x}(t), t)\vect{u}(\vect{x}(t), t)\right)dt + \vect{B}(\vect{x}(t), t)d\xi(t),
\end{equation*}
where $\vect{x}(t)\in\mathbb{R}^n$ denotes the state vector at time $t$, $\vect{u}(\vect{x}(t), t)\in\mathbb{R}^m$ denotes the control vector, and $ \xi(t)\in\mathbb{R}^m$ denotes Brownian motion. In many practical models, the disturbance (random perturbation) is directly added to the control input. Hence, for the ease of notation, let us suppose $g(\vect{x}(t), t) = \vect{B}(\vect{x}(t), t)$ as follows:
\begin{equation*}
    d\vect{x}(t) = \left(f(\vect{x}(t), t) + \vect{B}(\vect{x}(t), t)\vect{u}(\vect{x}(t), t)\right)dt + \vect{B}(\vect{x}(t), t)d \xi(t).
\end{equation*}

The PI methods aim to solve the finite time horizon optimal control problem with the following value function
\begin{equation}\label{eq:optimal_control_eqn}
\begin{aligned}
    &V(\vect{x}(t_0), t_0) = \min_{\vect{u}} \E \bigg[ \phi(\vect{x}_T, T) \\
    &+\int_{t_0}^T \left( q(\vect{x}(t), t) + \frac{1}{2} \vect{u}(\vect{x}(t), t)^\top \vect{R}(\vect{x}(t))\vect{u}(\vect{x}(t), t) \right) dt \bigg],
\end{aligned}
\end{equation}
where $\phi(\vect{x}_T, T)$ denotes the terminal cost with the terminal state $\vect{x}_T$, $q(\vect{x}(t), t)$ denotes the state dependent running cost and $\vect{R}$ is the weight matrix for the control effort.

The PI methods use the exponential transformation\footnote{This specific exponential transformation that was first introduced in~\cite{fleming1977exit} casts the stochastic optimal control into the particular structure where the determined optimal control depends on the parameter $\lambda$.} of the value function as
\begin{equation*}
    V(\vect{x}(t), t) = -\lambda\log(\Psi(\vect{x}(t), t))
\end{equation*}
with the Hamilton-Jacobi-Bellman (HJB) equation associated with the value function. The transformation makes the HJB a PDE that is linear in $\Psi$. The HJB is usually solved backward in time. However, the linearity allows us to reverse the direction of computation, replacing it by a diffusion process~\cite{kappen2005path, kappen2005linear}. Using the random trajectories calculated foward-pass as shown in Figure~\ref{fig:random_trajectories}, one can calculate the stochastic optimal control as a ratio of the expectations (for detailed derivations of the control law, see~\cite{kappen2005path, kappen2005linear}):
\begin{equation}
    \vect{u}^*(t) dt = \vect{u}(\vect{x}(t),t) dt + \frac{\E[e^{-(1/\lambda)S(\tau)}d \xi(t)]}{\E[e^{-(1/\lambda)S(\tau)}]},
\end{equation}
where $\vect{u}(\vect{x}(t),t)$ is the nominal control that is  initially designed  to be optimized and 
\begin{equation}\label{eq:stochastic_cost}
    S(\tau) = \phi(\vect{x}_T) + \int_{t_0}^T q(\vect{x}(t), t) dt
\end{equation}
is the cost-to-go of the state-dependent cost of a random trajectory $\tau$. The continuous-time trajectories are sampled as a \emph{discretized} system $\vect{x}_{t+1} =\vect{x}_t + d\vect{x}_t$ according to
\begin{equation}\label{eq:DT_Diffusion}
\begin{aligned}
    d\vect{x}_t &= \left(f(\vect{x}_t, t) + \vect{B}(\vect{x}_t, t)\vect{u}(\vect{x}_t, t)\right) \Delta t + \vect{B}(\vect{x}_t, t) \delta_t \sqrt{\Delta t} \\
    &= \tilde{f}(\vect{x}_t, t) \Delta t + \vect{B}(\vect{x}_t, t) \delta_t \sqrt{\Delta t},
\end{aligned}
\end{equation}
where $\delta_t$ is the Gaussian random vector with independent and identically distributed (i.i.d.) standard normal Gaussian random variables, i.e., $[\delta_t]_i\sim\mathcal{N}(0,1)$, and $\Delta t$ denotes the time step of the time-discretization using Euler–Maruyama method~\cite{platen2010numerical}. Then, the discretized version of the optimal control is 
\begin{equation}\label{eq:DT_optimal_control}
    \vect{u}_t^* = \vect{u}_t +  \frac{\E[e^{-(1/\lambda)S(\tau) }(\mathbf{\delta}_t/\sqrt{\Delta t}) ]}{\E[e^{-(1/\lambda)S(\tau) }]},
\end{equation}
where $\vect{u}_t$ is discrete time nominal control and
\begin{equation}\label{eq:DT_stochastic_cost}
    S(\tau) \approx \phi(\vect{x}_T) + \sum_{t=0}^{T-1} q(\vect{x}_t, t) \Delta t.
\end{equation}
The above expectations~\eqref{eq:DT_optimal_control} can be estimated with the cost of the sample trajectories using MC integration as
\begin{equation}
\label{eq:MC_approx_optimal_control}
    \hat{\vect{u}}_t^* \approx  \vect{u}_t  + \frac{\frac{1}{N}\sum_{n=1}^N\left(e^{-(1/\lambda)S^{(n)}} (\mathbf{\delta}_t/\sqrt{\Delta t}) \right)}{\frac{1}{N}\sum_{n=1}^N\left( e^{-(1/\lambda)S^{(n)}}\right)},
\end{equation}
where $N$ is the number of random trajectories sampled from the discrete time diffusion in~\eqref{eq:DT_Diffusion} and $S^{(n)}$ denotes the cost of the $n$th sample trajectory.  
\begin{remark}
Note the the exponential weighted sum in~\eqref{eq:MC_approx_optimal_control} is also called as soft-max function with the $\lambda$ as softness parameter. The function's behavior depends on the choice of $\lambda$.
For example,
with a very low value of $\lambda$, the PI simply chooses the sample trajectory with the minimum cost~\cite{kappen2005path}. 
\end{remark}

\subsection{Monte Carlo Integration of Expectation}
The expectations in~\eqref{eq:DT_optimal_control} are approximately calculated using the MC integration as in~\eqref{eq:MC_approx_optimal_control}.
The MC integration estimates the expectation
\begin{equation*}
    \E [ g(X) ] = \int_X g(x) p(x) dx
\end{equation*}
with the empirical mean as
\begin{equation*}
    \bar{g}_N = \frac{1}{N}\sum_{n=1}^N g(x^{(n)}),
\end{equation*}
where $x^{(n)} \sim p$ is the $n$th sample from the density $p$, and $N$ denotes the number of samples, and $g(x)$ is a scalar function.
A sampling complexity bound (from Chebyshev's inequality~\cite{markov1884certain}) on the MC integration is:
\begin{equation}\label{eq:chebysev_bounds}
    \Prob \left\{\vert \E [ g(X) ] - \bar{g}_N \vert \geq \epsilon \right\} \leq \rho:= \frac{\text{Var}[g(X)]}{N\epsilon^2},
\end{equation}
where $\epsilon>0$ denotes the error bound and the right hand side of the inequality denoted as $\rho$ is referred to as  risk probability of not satisfying the error bound.
Furthermore, for the bounded random variable within $[0, 1]$, i.e., $g(X)\in [0, 1]$, we know the following Hoeffding's inequality~\cite{hoeffding1994probability} 
\begin{equation}\label{eq:hoeffding_bounds}
    \Prob \left\{\vert \E [ g(X) ] - \bar{g}_N \vert \geq\epsilon \right\} \leq \rho:=2e^{-2N\epsilon^2}
\end{equation}
is applicable for the MC integration.
\subsection{Chi-square distribution}\label{sec:Noncentral_chi-square distribution}
We will use the chi-square distribution for calculating the expectation of the quadratic cost with the random state variable that has Gaussian distribution.
\begin{definition}
Let $(Z_1, Z_2,\dots, Z_k, \dots, Z_K)$ be $K$ independent normally distributed random variables with means $\mu_k$ and unit variance. Then the random variable $ W = \sum_{k=1}^K  Z_k^2$
is distributed according to the noncentral chi-squared distribution written as $W\sim \chi_K^2(\ell)$. It has two parameters: $K$ which specifies the number of degrees of freedom, and $\ell$ which is the sum of the squares of the means, i.e., $\ell=\sum_{k=1}^K\mu_k^2$. The mean of the noncentral chi-squared distribution is $K+\ell$. We denote the central chi-square distribution with degree $K$ as $\chi_K^2:=\chi_K^2(0)$. 
\end{definition}

\section{Results}

\subsection{Sampling complexity with the expectation of the cost}
The MC integration in~\eqref{eq:MC_approx_optimal_control} can be calculated fast using GPU. For example, 2500 sample trajectories with the discrete-time horizon of 150 steps were calculated with control frequency at 60 Hz~\cite{williams2016aggressive}. However, it has not been clear how many samples are needed for reasonable certainty of the MC integration. Hence, we are interested in the sampling complexity to see how well the MC optimal control approximates the discrete-time optimal control, given system dynamics and the cost functions. The following, easy to satisfy, standard  assumptions are made for our analysis.

\begin{assumption}\label{assumption:good choice_of_eps1}
We suppose that we choose the error bound of the MC integration, $\epsilon_1$, smaller than the expectation $\E[e^{-S(\tau)/\lambda}]$ in~\eqref{eq:DT_optimal_control}, i.e.,
\begin{equation}\label{eq:prior_eps1_choice}
    \epsilon_1 < \E[e^{-S(\tau)/\lambda}].
\end{equation}
\end{assumption}
\begin{assumption}\label{assumption:positive_running_cost}
We suppose that the terminal cost $\phi(\cdot)$ and the running cost $q(\cdot, \cdot)$ in~\eqref{eq:DT_stochastic_cost} are non-negative, therefore, $S(\tau)\geq0$.
\end{assumption}
\begin{proposition}\label{prop:sampling_complexity}
Under Assumptions~\ref{assumption:good choice_of_eps1} and~\ref{assumption:positive_running_cost}, the multiplicative and additive error between the optimal control $[\vect{u}_t^*]_i$ in~\eqref{eq:DT_optimal_control} and the MC integration $[\hat{\vect{u}}_t^*]_i$ in~\eqref{eq:MC_approx_optimal_control} are as follows:
\begin{equation}\label{eq:multi_add_error_bound}
    \left(1 - \frac{\epsilon_1}{\E[w]}\right)([\vect{u}_t^*]_i -\epsilon_2) \leq  [\hat{\vect{u}}_t^*]_i \leq \left(1 + \frac{\epsilon_1}{\E[w]}\right)([\vect{u}_t^*]_i +\epsilon_2)
\end{equation}
for all $i \in \{1,2,\dots,m\}$, where we define the weight random variable as $w:=\exp(-(1/\lambda)S(\tau))$ for the ease of notation. The error bounds $\epsilon_1$ and $\epsilon_2$ are from the following MC integration bounds for the number of sample trajectories $N$ that determine the risks $\rho_1$ and $\rho_2$ as:
\begin{equation} \label{eq:MC_bnd1}
    \Prob\{|\hat{E}_1 - \E[w]| \geq \epsilon_1\} \leq \rho_1:=2e^{-N\epsilon_1^2}, 
\end{equation}
\begin{equation} \label{eq:MC_bnd2}
    \Prob\left\{\left|\left[\hat{E}_2 - \E\left[\frac{w \delta_t}{\E[w]}\right]\right]_i\right| \geq \epsilon_2 \right\} \leq \rho_2:= \frac{1+\sqrt{2}}{  N\epsilon_2^2}\left(e^{\frac{2\E[S(\tau)]}{\lambda}}\right)
\end{equation}
where $\hat{E}_1$ denotes the denominator in~\eqref{eq:MC_approx_optimal_control}, i.e., $\hat{E}_1 := \frac{1}{N}\sum_n^N[ \exp(-(1/\lambda)S^{(n)})]$ and $\hat{E}_2$ denotes the sample cost weighted random input\footnote{For the ease of the notation, we drop the constant, $\sqrt{\Delta t}$ in~\eqref{eq:MC_approx_optimal_control} because you can simulate the trajectory or calculate the expectation of the cost considering this constant gain.}, i.e., $\hat{E}_2:=\frac{1}{N}\sum_n^N\left(\frac{w^{(n)}\delta_t^{(n)}}{\E[w]}\right)$.
\end{proposition}
\begin{proof}
To prove the above, we have the following lemmas.
\begin{lemma}\label{lem:MC_Approx_bnds}
The multiplicative and additive error between the optimal control $\vect{u}_t^*$ in~\eqref{eq:DT_optimal_control} and the MC integration $\hat{\vect{u}}_t^*$ in~\eqref{eq:MC_approx_optimal_control} are:
{\small
\begin{equation}\label{eq:multi_add_error_bound2}
    \left(1 - \frac{\epsilon_1}{\E[w]}\right)([\vect{u}_t^*]_i -\epsilon_2) \leq  [\hat{\vect{u}}_t^*]_i \leq \left(1 + \frac{\epsilon_1}{\E[w]}\right)([\vect{u}_t^*]_i +\epsilon_2)
\end{equation}
}
for all $i \in \{1,2,\dots,m\}$.
\end{lemma}
\begin{proof}
Rewriting the optimal control in~\eqref{eq:DT_optimal_control} with the short notation, multiplying and dividing by $\E[w]$ 
\begin{equation*}
    \vect{u}^*_t = \vect{u}_t+\frac{\E[w \delta_t]}{\E[w]} = \vect{u}_t+\left(\frac{\E[w]}{\E[w]}\right)\E\left[\frac{w \delta_t}{\E[w]}\right]= \vect{u}_t + AB,\\
\end{equation*}
where we denote $A:=\frac{\E[w]}{\E[w]}=1$; $B:=\E\left[\frac{w \delta_t}{\E[w]}\right]$.
The MC integration counter part is 
\begin{equation*}
\begin{aligned}
    \hat{\vect{u}}^*_t &= \vect{u}_t+\frac{\frac{1}{N}\sum_n^Nw^{(n)}\delta^{(n)}_t}{\frac{1}{N}\sum_n^N w^{(n)}} = \vect{u}_t+\left(\frac{\E[w]}{\hat{E}_1}\right)\frac{1}{N}\sum_n^N\left(\frac{w^{(n)}\delta^{(n)}_t}{\E[w]}\right)  \\
    &= \vect{u}_t + \hat{A}\hat{B},
\end{aligned}
\end{equation*}
where we denote $\hat{A}:=\frac{\E[w]}{\frac{1}{N}\sum_n^Nw^{(n)}}$; $\hat{B}:=\frac{1}{N}\sum_n^N\left(\frac{w^{(n)}\delta_t^{(n)}}{\E[w]}\right)$.
The error of the approximation is $\tilde{\vect{u}}_t = \hat{\vect{u}}^*_t - \vect{u}_t^* = AB - \hat{A}\hat{B}$.
With elementary calculations with~\eqref{eq:MC_bnd1}, we found the following error bounds of  $\hat{A}$ from $A=1$ as
\begin{equation*}
    1 - \frac{\epsilon_1}{\hat{E}_1} \leq \hat{A} \leq 1 + \frac{\epsilon_1}{\hat{E}_1} \quad \text{and} \quad 1 - \frac{\epsilon_1}{\E[w]} \leq \hat{A} \leq 1 + \frac{\epsilon_1}{\E[w]}.
\end{equation*}
And the other error bound of $\hat{B}$ from $B$ is as $[B]_i - \epsilon_2 \leq [\hat{B}]_i \leq [B]_i + \epsilon_2$,
for all $i \in \{1,2,\dots,m\}$.
Using~\eqref{eq:prior_eps1_choice}, we know that  both limits of the bound for $\hat{A}$ are positive. So, multiplying the bounds for $\hat{A}$ and $\hat{B}$ results in~\eqref{eq:multi_add_error_bound2}.
\end{proof}

To use the sampling complexity bounds in Lemma~\ref{lem:MC_Approx_bnds}, we need to study the variances $\Var(w)$ and $\Var(\frac{w\delta}{\E[w]})$. A bound on $\Var(w)$ is stated in the following lemma.
\begin{lemma}\label{lem:bnd_var_w}
$\Var(w) \leq (1-\E[w])\E[w] \leq \E[w]\leq 1$
\end{lemma}
\begin{proof}
 Note that $w = e^{-\frac{S(\tau)}{\lambda}}$ and $S \geq 0$. So, $w \in [0,1]$. Then using the known properties of bounded random variable and its variance the above follows.
\end{proof}
A bound on $\Var(\frac{w\delta_t}{\E[w]})$ is stated as in the following lemma.
\begin{lemma}\label{lemma:inverse_of_expectation}
The bound on the variance of the weighted control depends on the expectation of the cost $\E[S(\tau)]$  as:
\begin{equation*}
    \Var\left[\frac{w[\delta_t]_i}{\E[w]}\right] =\frac{1}{\E[w]^2} \Var[w[\delta_t]_i] \leq \frac{1+\sqrt{2}}{\left(e^{\frac{-\E[S(\tau)]}{\lambda}}\right)^2}  ,
\end{equation*}
where $[\delta_t]_i$ denotes the $i$th element of the vector-valued random variable $\delta_t \sim \mathcal{N}(\bar{\vect{0}}, \vect{I})$ and $w:= e^{-\frac{S(\tau)}{\lambda}}$.
\end{lemma}
\begin{proof}
Using properties of the product of random variables with $w$ and $[\delta_t]_i \sim \mathcal{N}(0, 1)$, we have the following bound of the variance
 \begin{equation*}
    \begin{aligned}
        \Var(w \delta(i)) &\leq \left(\Var(w^2)\Var([\delta_t]_i^2)\right)^{1/2} +\E(w^2)\E([\delta_t]_i^2) \\
        & = \left(2\Var(w^2)\right)^{1/2} +\E(w^2)  \leq \sqrt{2} + 1,
    \end{aligned}
\end{equation*}
where we used \eqref{eq:bound_var_product} of Appendix for the first inequality,  $\Var([\delta_t]_i^2)=2$  and $\E[[\delta_t]_i^2]=1$ for $[\delta_t]_i^2 \sim \chi_1(0)$, and the last inequality follows using the properties on the moments of the bounded random variable, i.e., $w^2\in[0,1]$ in Lemma~\ref{lem:bnd_var_w}, e.g., $\E(w^2) \leq 1$ and $\Var(w^2) \leq 1$.

Applying the above bound to $\Var\left[\frac{w\delta_t}{\E[w]}\right]$, we arrive at the claimed bound as
\begin{equation}\label{eq:lemma2}
\Var\left[\frac{w\delta_t}{\E[w]}\right] \leq \frac{1+\sqrt{2}}{\E[w]^2} \leq \frac{1+\sqrt{2}}{\left(e^{\frac{-\E[S(\tau)]}{\lambda}}\right)^2}  ,
\end{equation}
where the second inequality follows from Jensen's inequality with the convex function (exponential) $w = e^{-\frac{S(\tau)}{\lambda}}$, i.e.,
\begin{equation*}
e^{-\E[S(\tau)]/\lambda} \leq \E[w]=\E[e^{-\frac{S(\tau)}{\lambda}}] \quad \text{and} \quad \frac{1}{\E[w]} \leq \frac{1}{e^{-\E[S(\tau)]/\lambda}}
\end{equation*}
\end{proof}
The sampling complexity in Proposition~\ref{prop:sampling_complexity} follows by combining the previous Lemmas.
\end{proof}
\begin{remark}
Proposition~\ref{prop:sampling_complexity} is useful for the case where the expectation $\E[S(\tau)]$ can be explicitly calculated given nominal control and the  linear dynamics because  Jensen's inequality converts $\E[e^{-\frac{S(\tau)}{\lambda}}]$ to    $e^{-\E[S(\tau)]/\lambda}$. 
\end{remark}
The proposition states that an error bound of the resultant control $\hat{u}_t^*$ from the desired optimal control $u_t^*$ is expressed with a multiplicative error bound $\epsilon_1$ and an additive error bound $\epsilon_2$ as in~\eqref{eq:multi_add_error_bound}. Given $\epsilon_1$ and $\epsilon_2$ in Proposition~\ref{prop:sampling_complexity}, we can arbitrarily decrease the risk of error $\rho_1$ in~\eqref{eq:MC_bnd1} corresponding to $\epsilon_1$ and the risk of error $\rho_2$ in~\eqref{eq:MC_bnd2} corresponding to $\epsilon_2$ by increasing the number of samples for the MC integration. The risks $\rho_1$ and $\rho_2$ decrease as $N$ increases, i.e., $\rho_1= e^{-N\epsilon_1^2}$ and $\rho_2= O(1/N)$.


For the case where the expectation of the cost $\E[S]$ cannot be calculated, we can still assess uncertainties using Lemma~\ref{lem:MC_Approx_bnds} with the empirical mean $\hat{E}_1$ as in the following corollary.    
\begin{corollary}\label{corollary:sampling_complexity_with_emphrical}
Under Assumptions~\ref{assumption:good choice_of_eps1} and~\ref{assumption:positive_running_cost}, the second MC error bound in~\eqref{eq:MC_bnd2} of Proposition~\ref{prop:sampling_complexity} becomes
\begin{equation}
    \Prob\left\{\left|\hat{E}_2 - \E\left[\frac{w \delta_t}{\E[w]}\right]\right| \geq \epsilon_2 \right\} \leq \frac{1+\sqrt{2}}{  N\epsilon_2^2}\left(\frac{1}{\hat{E}_1 -\epsilon_1}\right)^2,
\end{equation}
where $\epsilon_1$ is from the first MC error bound in ~\eqref{eq:MC_bnd1}, and $\hat{E}_1$ and $\hat{E}_2$ are the empirical means from Proposition~\ref{prop:sampling_complexity}.
\end{corollary}
\begin{proof}
Using the first inequality on the variance in~\eqref{eq:lemma2}, we have
\begin{equation*}
    \Var\left[\frac{w\delta_t}{\E[w]}\right] \leq \frac{1+\sqrt{2}}{\E[w]^2} 
\end{equation*}
with the following inequality from~\eqref{eq:MC_bnd1} 
\begin{equation*}
    \frac{1}{(\hat{E}_1+\epsilon_1)^2} \leq \frac{1}{(\E[w])^2} \leq \frac{1}{(\hat{E}_1-\epsilon_1)^2}.
\end{equation*}
Then the above follows using the Chebysev inequality.
\end{proof}

\subsection{Calculation of the expectation with linear system and a class of cost functions}

Consider the following linear time-varying (LTV) discrete-time stochastic system that belongs to the discrete time diffusion in~\eqref{eq:DT_Diffusion}:
\begin{equation}\label{eq:model}
    \vect{x}_{t+1} = \vect{A}_t \vect{x}_{t} + \vect{B}_t (\vect{u}_t + \vect{\delta}_t),
\end{equation}
where $\vect{x}_t\in\mathbb{R}^n$, $\vect{u}_t\in\mathbb{R}^m$ and $\vect{\delta}_t\in\mathbb{R}^m$ are the state vector, the nominal control input, and the random perturbation with independent and identically distributed (i.i.d.) normal distribution, i.e.,  $\delta_t \sim \mathcal{N}(\bar{\vect{0}}, \vect{I})$. System matrices $\vect{A}_t$ and $\vect{B}_t$ are known and bounded with appropriate dimensions.


We consider the following class of cost functions as described in the following assumptions.
\begin{assumption}\label{A1}
For the the running cost  in~\eqref{eq:stochastic_cost}, we assume  quadratic cost with an extension to consider the constraint as indicator function 
\begin{equation}\label{eq:a_class_of_running_cost}
    q(\vect{x}_t, t)= (\vect{x}_t - \vect{x}_\text{tgt})^\top \vect{Q} (\vect{x}_t- \vect{x}_\text{tgt}) + \omega_C\mathds{1}_{C_t}(\vect{x}_t),  
\end{equation}
where $\vect{x}_\text{tgt}$ denotes target state value, e.g., goal position and $\mathds{1}_{C_t}$ denotes indicator function such that
\begin{equation*}
\mathds{1}_{C_t}(\vect{x}) :=
\left\{
  \begin{array}{@{}ll@{}}
    1, & \text{if} \quad \vect{x} \in C_t\\
    0, & \text{otherwise},
  \end{array}\right.
\end{equation*}
where $C_t$ is the set to avoid and $\omega_C>0$ is the penalty weight. Also, we assume that  $\vect{Q}$ is a constant positive definite matrix. 
For the terminal cost $\phi(\vect{x}_T)$ in~\eqref{eq:optimal_control_eqn}, we consider the following quadratic cost as
\begin{equation}
 \phi(\vect{x}_T) =    (\vect{x}_T - \vect{x}_\text{tgt})^\top \vect{Q}_T (\vect{x}_T- \vect{x}_\text{tgt}),
\end{equation}
where $\vect{x}_T$ denotes the terminal state and $\vect{Q}_T>0$ denotes the terminal cost matrix.
\end{assumption}
For the above class of cost functions in Assumption~\ref{A1} and LTV system in~\eqref{eq:model} with Gaussian additive noise $\delta_t \sim \mathcal{N}(\bar{\vect{0}}, \vect{I})$ and nominal control $\vect{u}_t$, we can calculate the expectation of the cost $S(\tau)$ in~\eqref{eq:DT_stochastic_cost}, using the properties of the linear transformation of Gaussian random variable. The mean $\hat{\vect{x}}_t$ and covariance $\vect{P}_t$ of the trajectories of the state distribution, i.e., $\vect{X}_t\sim\mathcal{N}(\hat{\vect{x}}_t,\vect{P}_t)$, follow
\begin{equation}\label{eq:random_trajectories}
\begin{aligned}
    \hat{\vect{x}}_{t+1} &= \vect{A}_t \hat{\vect{x}}_{t} + \vect{B}_t \vect{u}_t, \quad  \hat{\vect{x}}_{0} = \vect{x}_{0}\\
    \vect{P}_{t+1} &= \vect{A}_t \vect{P}_{t} \vect{A}_t^\top + \vect{B}_t \vect{B}_t^\top, \quad \vect{P}_{0} = \text{diagonal}(\overline{\vect{0}}).
\end{aligned}
\end{equation}
Note that the above trajectories can be calculated by recursively updating the state distribution $T$ times using~\eqref{eq:random_trajectories}. At each time $t$, $(\hat{\vect{x}}_t,\vect{P}_t)$ determines the running cost and the terminal cost.
\begin{lemma}\label{lem:E_quad_cost}
For the distribution of the quadratic state cost, we have
\begin{equation*}
    (\vect{X}_t - \vect{x}_\text{tgt})^\top \vect{Q} (\vect{X}_t- \vect{x}_\text{tgt}) = \sum_{i=1}^m \lambda_i Z_i,
\end{equation*}
where $\vect{X}_t\sim\mathcal{N}(\hat{\vect{x}}_t,\vect{P}_t)$ is a state distribution, and $Z_i$ follow the noncentral Chi-square distribution below
\begin{equation*}
    Z_i\sim\chi_1([\mu]_i^2), \quad \mu=\tilde{\vect{V}}^{-1}\vect{V}(\hat{\vect{x}}_t-\vect{x}_{tgt})
\end{equation*}
and the eigenvalues $\lambda_i$ are from $\vect{V}^\top\vect{P}_t\vect{V}$, where $\vect{V}^\top \vect{V} = \vect{Q}$ and $\tilde{\vect{V}}^\top \tilde{\vect{V}} = \vect{V}^\top\vect{P}_t\vect{V}$.
The expectation of the quadratic cost is
\begin{equation*}
    \E[(\vect{X}_t - \vect{x}_\text{tgt})^\top \vect{Q} (\vect{X}_t- \vect{x}_\text{tgt})] = \sum_{i=1}^m \lambda_i(1+[\mu]_i^2).
\end{equation*}
\end{lemma}
\begin{proof}
The above follows after a few changes of variables and eigen-decompositions of the  matrices. Let $\tilde{\vect{Z}}:=\vect{V}(\vect{X}_t - \vect{x}_{tgt})$ with the eigen-decomposition  of the positive definite matrix $Q$, $\vect{V} \vect{V}^\top = \vect{Q}$ so that $(\vect{X}_t - \vect{x}_\text{tgt})^\top \vect{Q} (\vect{X}_t- \vect{x}_\text{tgt}) = \tilde{\vect{Z}}^\top \tilde{\vect{Z}}$,
where $\tilde{\vect{Z}}\sim\mathcal{N}\left(\vect{V}(\hat{\vect{x}}_t-\vect{x}_{tgt}),\vect{V}^\top\vect{P}_t\vect{V}\right)$.
Another change of the variable $\vect{Z}:=\tilde{\vect{V}}^{-1}\tilde{\vect{Z}}$ with the eigen-decomposition $\tilde{\vect{V}}^\top \tilde{\vect{V}} = \vect{V}^\top\vect{P}_t\vect{V}$ gives
$
     \tilde{\vect{Z}}^\top \tilde{\vect{Z}}=\vect{Z}^\top \tilde{\vect{V}}^\top\tilde{\vect{V}} \vect{Z} , \quad \vect{Z}\sim\mathcal{N}(\tilde{\vect{V}}^{-1}\vect{V}\left(\hat{\vect{x}}_t-\vect{x}_{tgt}),\vect{I}\right).
$
Let $Z_i=[\vect{Z}]_i\sim\mathcal{N}([\mu]_i, 1)$~and use the mean of the noncentral Chi-squared distribution in~\ref{sec:Noncentral_chi-square distribution} to get the desired result.
\end{proof}
Now, we calculate the expectation of the indicator function in~\eqref{eq:a_class_of_running_cost}. Given a general compact set as an obstacle in Figure~\ref{fig:obs_gaussian_state}, integrating the multivariate Gaussian distribution over the compact set for collision risk calculation does not have closed form solution. So, we use a conservative set as the constraint set $C_t$ at each time $t$ to consider the risk of collision to the obstacle $O$ as stated in the following assumption.
\begin{assumption}\label{A:Convex hull} Suppose
a static obstacle set $O\subseteq{\mathbb{R}^m}$ is a convex hull. Furthermore, we are given with the closest vertex $\vect{c}_t^*\in O$. Then the conservative set $C_t$ is the outer of the circle that is centered at $\hat{\vect{x}}_t$ and passing $\vect{c}_t^*$ with the shape determined by $\vect{P}_t$ as shown in Figure~\ref{fig:obs_gaussian_state}. Note that $\Prob\{\vect{X}_t \in O \} \leq \Prob\{\vect{X}_t \in C_t \}$ because $O \subseteq C_t$.
\begin{equation*}
    C_t = \{\vect{X}_t: (\vect{X}_t-\hat{\vect{x}}_t)^\top (\vect{X}_t-\hat{\vect{x}}_t) \geq (\vect{c}_t^*-\hat{\vect{x}}_t)^{\top} (\vect{c}_t^*-\hat{\vect{x}}_t)  \}
\end{equation*}
\end{assumption}
\begin{figure}[thpb]
\centering
 \includegraphics[width=0.35\textwidth]{ 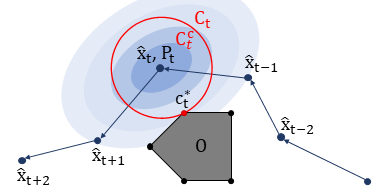}  
 \caption{Collision risk with the state distribution $\mathcal{N}(\hat{\vect{x}}_t, \vect{P}_t)$.}
\medskip
\label{fig:obs_gaussian_state}
\vspace{-0.1in}
\end{figure}
\begin{lemma}\label{lemma:indicator_cost}
Under Assumption~\ref{A:Convex hull}, the expectation of the indicator is 
\begin{equation*}
    \E[\mathds{1}_{C_t}(\vect{X}_t)] = \Prob\{\vect{X}_t \in C_t \} = 1- F_{\chi^2_d}((\vect{c}^*-\hat{\vect{x}}_t)^\top \vect{P}_t^{-1}(\vect{c}^*-\hat{\vect{x}}_t)),
\end{equation*}
where $F_{\chi^2_d}(\cdot)$ denotes cumulative distribution function of the chi-squared distribution  $\chi^2_d$.
\end{lemma}
\begin{proof}
For the ease of notation, without loss of the generality, we can translate the coordinates by $\hat{\vect{x}}_t$ and consider the distribution $\vect{X}_t\sim \mathcal{N}(\bar{\vect{0}}, \vect{P}_t)$ and keep the notation of the closest vertex $\vect{c}^*_t$. Also, we drop the time index $t$ because $t$ is fixed within this proof. Then the set $C$ is written with the following equivalent inequalities:
\begin{equation*}
\begin{aligned}
    C &= \{\vect{X}: \vect{X}^\top \vect{X} \geq \vect{c}^{*\top} \vect{c}^*  \} 
      = \{\vect{Z}: \vect{Z}^\top \vect{V}^\top \vect{V} \vect{Z} \geq \vect{c}^{*\top} \vect{c}^*  \} \\
      &= \{\vect{Z}: \vect{Z}^\top \vect{Z} \geq \vect{c}^{*\top} \vect{P}^{-1} \vect{c}^*  \}\\
      &= \{W: W=\sum_{i=1}^n \vect{Z}(i)^2 \geq \vect{c}^{*\top}\vect{P}^{-1} \vect{c}^*  \}
\end{aligned}
\end{equation*}
where $\vect{Z} = \vect{V}^{-1}\vect{X}$ with the eigen-decomposition $\vect{V} \vect{V}^\top = \vect{P}$  makes the univariate Gaussian $\vect{Z}\sim\mathcal{N}(\bar{\vect{0}}, \vect{I})$. Then $W=\sum_{i=1}^n Z_i^2$ follows $\chi_{n}^2$ distribution. For the equivalence of the second and third inequality, we used the vector version of the following fact with a scalar, i.e., $z^2v^2\geq c^2 \Leftrightarrow z^2\geq c^2/v^2$.
\end{proof}

\subsection{Growth of the co-variance  matrices  and the expected cost for unstable dynamics over time-horizon.}
Recall the covariance matrix update equation in~\eqref{eq:random_trajectories} below
\begin{equation*}
    \vect{P}_{t+1} = \vect{A}_t \vect{P}_{t} \vect{A}_t^\top + \vect{B}_t \vect{B}_t^\top, \quad \vect{P}_{0} = \text{diagonal}(\bar{\vect{0}}).
\end{equation*}
Applying basic matrix norm property to the above, we have
\begin{equation}\label{eq:matrix_norm_inequalities}
    \|\vect{P}_{t+1}\|_2 \geq  \sigma_\text{min}^2(\vect{A}_t)\|\vect{P}_{t}\|_2 + \sigma_\text{min}^2(\vect{B}_t)
\end{equation}
Recursively applying the above update law in~\eqref{eq:random_trajectories} and the corresponding matrix norm inequalities in~\eqref{eq:matrix_norm_inequalities}, we have the following: for $t>1$,
\begin{equation*}
    \vect{P}_{t} =  \vect{B}_{t-1} \vect{B}_{t-1}^\top + \sum_{\tau=0}^{t-2}\left(\Pi^{t-1}_{i=\tau+1}\vect{A}_i\right)\vect{B}_\tau \vect{B}_\tau^\top\left(\Pi^{t-1}_{i=\tau+1}\vect{A}_i\right)^\top 
\end{equation*}
with $\vect{P}_{1} =  \vect{B}_{0} \vect{B}_{0}^\top $; and the bound
\begin{equation*}
\begin{aligned}
    \|\vect{P}_{t+1}\| &\geq \sigma^2_\text{min}(\vect{B}_{t-1})+  \sum_{\tau=0}^{t-2}\sigma^2_\text{min}\left(\Pi^{t-1}_{i=\tau+1}\vect{A}_i\right)\sigma^2_\text{min}(\vect{B}_\tau)
\end{aligned}
\end{equation*}
For a special case of LTV systems with all singular values of $\vect{A}_t$ greater than one, it is easy to see that the lower bound increases exponentially over time-horizon.

\begin{corollary}\label{corllary:lambda_and_q_cost}
Consider the expectation of the quadratic cost in Lemma~\ref{lem:E_quad_cost} as below
\begin{equation*}
    \E[(\vect{X}_t - \vect{x}_\text{tgt})^\top \vect{Q} (\vect{X}_t- \vect{x}_\text{tgt})] = \sum_{i=1}^n \lambda_i(1+\mu(i)),
\end{equation*}
where the eigenvalues $\lambda_i$ are from $\vect{V}^\top\vect{P}_t\vect{V}$, where $\vect{V}^\top \vect{V} = \vect{Q}$. Then we have
\begin{equation*}
   \lambda_\text{min}(\vect{Q})\lambda_\text{min}(\vect{P}_t)  \leq \E[(\vect{X}_t - \vect{x}_\text{tgt})^\top \vect{Q} (\vect{X}_t- \vect{x}_\text{tgt})]
\end{equation*}
\end{corollary}
\begin{proof}
Using the eigen-decomposition of $\vect{Q}=\vect{V}_\vect{Q}^\top \Lambda_\vect{Q} \vect{V}_\vect{Q} = (\vect{V}_\vect{Q}^\top \Lambda_\vect{Q}^{1/2})(\Lambda_\vect{Q}^{1/2}\vect{V}_\vect{Q}) = \vect{V}^\top\vect{V} $ and $\vect{P}_t =\vect{V}_\vect{P}^\top \Lambda_\vect{P} \vect{V}_\vect{P}$, we have $
    \vect{V}^\top\vect{P}_t\vect{V} = (\vect{V}_\vect{Q}^\top \Lambda_\vect{Q}^{1/2})\vect{V}_\vect{P}^\top \Lambda_\vect{P} \vect{V}_\vect{P}(\Lambda_\vect{Q}^{1/2}\vect{V}_\vect{Q})
$.
Hence, $\lambda_\text{min}(\vect{Q})\lambda_\text{min}(\vect{P}_t) \leq \lambda_i$.
\end{proof}
For an LTV system, using the above corollary, we can estimate the growth of the expectation of the quadratic cost. To deduce insightful growth estimates with simple assumptions, we consider the linear time-invariant system subject to the following assumption. 

\begin{assumption}\label{assumption:Unstable_LTI}
Suppose the system in~\eqref{eq:random_trajectories} is linear time-invariant with $\vect{A}$ having an \emph{unstable} eigenvalue, i.e., $|\lambda_1(\vect{A})| > 1$ with full-rank and  $\vect{B}$ is a full-rank square matrix. 
\end{assumption}
\begin{corollary}\label{corollary:exploding_cost_unstable}
Under Assumption~\ref{assumption:Unstable_LTI}, the expectation of the quadratic cost in Corollary~\ref{corllary:lambda_and_q_cost} increases exponentially over time-horizon, i.e. there exists $\sigma > 0$ such that
\begin{equation*}
    \sigma |\lambda_1(\vect{A})|^{2t} \leq  \E[(\vect{X}_t - \vect{x}_\text{tgt})^\top \vect{Q} (\vect{X}_t- \vect{x}_\text{tgt})].
\end{equation*}
Then the calculated upper bound of the variance of the weighted control in Proposition~\ref{prop:sampling_complexity}, given by
\begin{equation*}
\Var\left[\frac{w\delta_t}{\E[w]}\right] \leq \frac{1+\sqrt{2}}{  N\epsilon_2^2}\left(e^{\frac{2\E[S]}{\lambda}}\right)
\end{equation*}
increases doubly exponentially over time-horizon, i.e., $e^{C|\lambda_1(\vect{A})|^{2t}}$.
\end{corollary}
\begin{proof}
It is easy to see that for the special case $\vect{B} \vect{B}^\top = \sigma \vect{I}$ with $\sigma > 0$, we have
\begin{equation*}
\vect{P}_{t+1} = \vect{A} \vect{P}_{t} \vect{A}^\top + \sigma\vect{I}, \quad \vect{P}_{0} = \text{diagonal}(\bar{\vect{0}})
\end{equation*}
and its recursive sum is
\begin{equation*}
\vect{P}_{t} = \sigma \sum_{\tau=1}^{t-1} \vect{A}^{\tau}(\vect{A}^\top)^\tau  + \sigma\vect{I}, \quad \vect{P}_{0} = \text{diagonal}(\bar{\vect{0}})
\end{equation*}
for $t>1$.
The eigenvalues of $\vect{A}^\tau(\vect{A}^\top)^\tau$ are the $2\tau$th power of the eigenvalues, i.e. $\lambda_i(\vect{A})^{2\tau}$. For the unstable eigenvalue $|\lambda_1(\vect{A})|>1$, $\lambda_1(\vect{A})^{2\tau}$ grows exponentially. To convey the special case for the general case, we use the following basic properties of the quadratic term:
$\vect{A}^\top (\vect{P} + \vect{P}')\vect{A} \geq \vect{A}^\top (\vect{P})\vect{A}$ for all $\vect{P}>0$ and $\vect{P}'>0$; $\vect{A}^\top (\vect{P}')\vect{A} \geq \vect{A}^\top (\vect{P})\vect{A}$ for all $\vect{P}'\geq\vect{P}>0$ and $\vect{P}'>0$. Since we can find $\sigma > 0$ such that $\vect{B} \vect{B}^\top \geq \sigma \vect{I}$, we can use the exponential growth for the special case to the general $\vect{B} \vect{B}^\top$.
\end{proof}
We end this section by a remark that Lemma~\ref{lem:E_quad_cost} and Lemma~\ref{lemma:indicator_cost} can be used to calculate $\E[S]$ in~\eqref{eq:DT_stochastic_cost}. Once we have $\E[S]$, we can apply it with the sampling complexity bound~\eqref{eq:MC_bnd2} of Proposition~\ref{prop:sampling_complexity}. For example, we can apply this approach to the recent work~\cite{yin2021improving} that uses feedback control to steer covariance dynamics $\vect{P}_t$ for improved sample efficiency. The work empirically shows improved sample efficiency through the greater performance of the proposed one compared to conventional MPPI that uses the same number of samples. In contrast, we can use our sampling complexity results to show a lower number of samples is required to satisfy the error bound when feedback control is used, compared to the MPPI.

\section{Numerical Examples}
\subsection{Simulations on unmanned aerial vehicle}

We consider a double integrator model to simulate the unmanned aerial vehicle (UAV) control problem:
{\small
\begin{align*}
    \vect{x}_{t+1}=
    \left[
    \begin{array}{cccc}
     1 & 0 & 0.1 & 0   \\
     0 & 1 &  0 & 0.1 \\
     0 & 0 &  1+a & 0   \\
     0 & 0 &  0 & 1+a
    \end{array}
    \right] \vect{x}_t +
    \left[
    \begin{array}{cc}
     0  & 0  \\
     0  & 0  \\
     0.1 & 0  \\
     0  & 0.1 \\
    \end{array}
    \right](\vect{u}_t+\delta_t) ,
\end{align*}
}
where $a\in[-0.5, 0.5]$ is a model parameter to set the stability of the system for the experiment, e.g., if $a=0.5$ then the system matrix $\vect{A}$ has eigenvalues 1, 1, 1.5, and 1.5, and $\vect{x}_t \in \mathbb{R}^4$ represents the horizontal coordinate, vertical coordinate, horizontal velocity, and vertical velocity. The input $\vect{u}_t \in \mathbb{R}^2$ consists of horizontal and vertical accelerations and $\delta_t\sim \mathcal{N}(\bar{\vect{0}}, \vect{I})$ denotes the random perturbation corresponding to the input. The initial condition is $\vect{x}_0 = [0,0,0,0]^\top$. The objective of the UAV navigation is to reach  to the goal position at $(8,8)$ with zero velocity, i.e., $\vect{x}_\text{tgt} = [8,8,0,0]^\top$, while avoiding collisions with the rectangular obstacle with distance margin at 0.2 as shown in Figure~\ref{fig:traj_given_varying_stability}. We set the cost parameters as $\vect{Q}=\vect{I}$ and the constraint weight as $\omega_C=100$ for the cost equation in~\eqref{eq:a_class_of_running_cost}. The MPPI parameters are as follows: Number of sample trajectories $N=10000$; $\lambda=1$; time horizon $T=40$ for the model predictive control; total simulation time step is 7000 steps. The nominal control was set to zero, i.e., $\vect{u}_t = \bar{\vect{0}}$.

Since the MPC time horizon $T$ is relatively short compared to the entire trajectory, we can easily guess an ideal MPC trajectory that is moving to the target position with the minimum path length. As shown in Figure~\ref{fig:traj_given_varying_stability}, the one with stable velocity dynamics follows such minimum path trajectories. For the other stability conditions, we can see that the performance gets deteriorated as $a$ increases.  
\begin{figure}[thpb]
\centering
 \includegraphics[width=0.3\textwidth]{ 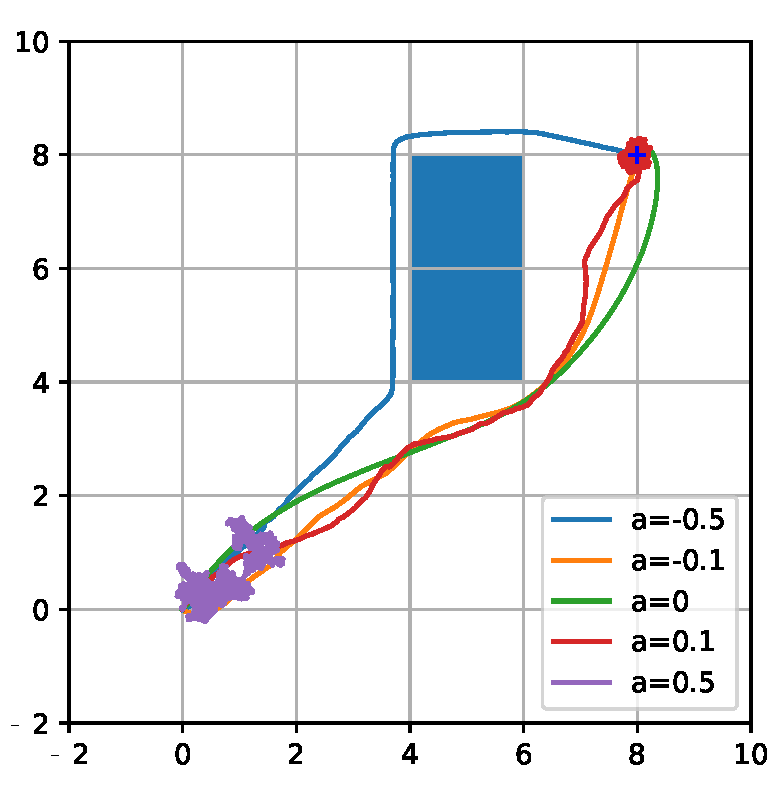}
 \caption{Flight path depending on the stability.}
\medskip
\label{fig:traj_given_varying_stability}
\vspace{-0.15in}
\end{figure}
In Corollary~\ref{corollary:exploding_cost_unstable}, we showed that for the unstable system the expectation of the cost increases exponentially over time horizon when we use the PI methods. Figure~\ref{fig:variance_of_control_samples} shows the sample variance of the random variable $\frac{w \delta_t}{\E[w]}$ in~\eqref{eq:MC_bnd2}. As we have seen in Corollary~\ref{corollary:exploding_cost_unstable}, the variance increases as the time horizon increases and the system becomes unstable. Also, comparing Figure~\ref{fig:variance_of_control_samples} and Figure~\ref{fig:mean_of_the_sample_weight}, the variance's dependence as the inverse of the expectation of the weight in Lemma~\ref{lemma:inverse_of_expectation} is realized. 
\begin{figure}[thpb]
\vspace{-0.15in}
\centering
\begin{subfigure}{.24\textwidth}
  \centering
  \includegraphics[width=.95\linewidth]{ 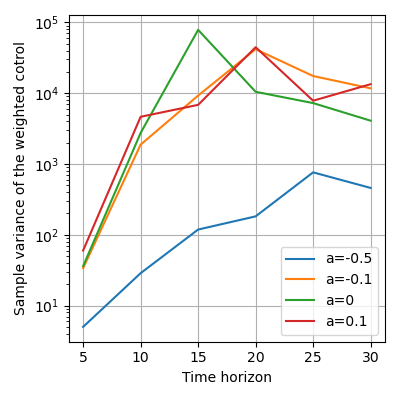}
  \caption{Variance of weighted control}
  \label{fig:variance_of_control_samples}
\end{subfigure}%
\begin{subfigure}{.24\textwidth}
  \centering
  \includegraphics[width=.95\linewidth]{ 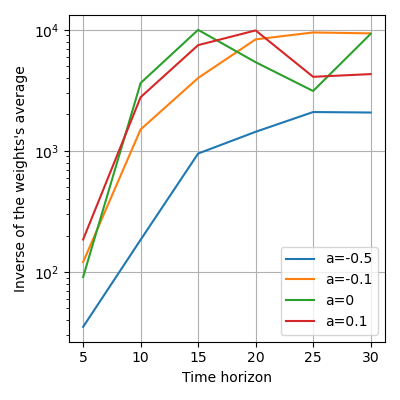}
  \caption{Inverse of average weights}
  \label{fig:mean_of_the_sample_weight}
\end{subfigure}
\caption{Empirical variance of the weighted control and the inverse of the average weight for varying stability and time-horizon.}
\label{fig:test}
\vspace{-0.15in}
\end{figure}
\subsubsection{Use of the sampling complexity}
We set the desired error bounds  $\epsilon_1$, $\epsilon_2$, the allowable risks of failure of $\rho_1$, and $\rho_2$ to the values shown in Table~\ref{tab:sample_uav}, to use the inequality from Proposition~\ref{prop:sampling_complexity} as below
{\small
\begin{equation*}
    \left(1 - \frac{\epsilon_1}{\E[w]}\right)([\vect{u}_t^*]_i -\epsilon_2) \leq  [\hat{\vect{u}}_t^*]_i \leq \left(1 + \frac{\epsilon_1}{\E[w]}\right)([\vect{u}_t^*]_i +\epsilon_2).
\end{equation*}
}
\vspace{-0.15in}
\begin{table}[h]
\tiny
\caption{Required number of sample $max(N_1, N_2)$ given $\epsilon_1=0.01$, $\epsilon_2=0.1$, $\rho_1=0.05$, and $\rho_2=0.05$}
\label{tab:sample_uav}
\begin{center}
\begin{small}
\begin{sc}
\begin{tabular}{crrcc}
\toprule
$a$ & $T$ & $N_2$    & $\frac{\epsilon_1}{\E[w]}$   &$N_1$\\
\midrule
\multirow{3}{*}{-0.5}  &50   & 5789  &  0.011   & \multirow{3}{*}{18444}\\
                       &100   & 6270  &  0.011   & \\
                       &150   & 6793  &  0.012   & \\

\midrule
\multirow{3}{*}{-0.1}  &50   & 22203  &  0.021   & \multirow{3}{*}{18444}\\
                       &100   & 148140  &  0.055   & \\
                       &150   & 1058089  &  0.148   & \\

\midrule
\multirow{3}{*}{0}   &50   & 1899513595      &  1.983   & \multirow{3}{*}{18444}\\
                       &100   & $5.4*10^{37}$     &  $1*10^{15}$   &\\
                       &150   & $2.4*10^{128}$  &  $2*10^{60}$   &\\
                       
\midrule
\multirow{3}{*}{0.1}   &50   &  \multirow{3}{*}{overflow}      &  \multirow{3}{*}{overflow}   & \multirow{3}{*}{18444}\\
                       &100   &      &   &\\
                       &150   &   &     &\\

\bottomrule
\end{tabular}
\end{sc}
\end{small}
\end{center}
\vskip -0.2in
\vspace{-0.05in}
\end{table}

\subsection{Simulations on unmanned ground vehicle}

The unmanned ground vehicle (UGV) simulations use the simple car kinematic model described by:
{\small
\begin{equation}\label{eq:simple_car}
    \begin{bmatrix}
    p^x_{t+1} \\
    p^y_{t+1} \\
    \theta_{t+1} \\
    \phi_{t+1}
    \end{bmatrix}
    =
    \begin{bmatrix}
    p^x_{t} \\
    p^y_{t} \\
    \theta_{t} \\
    \phi_{t}
    \end{bmatrix}
    +
    \begin{bmatrix}
    (\cos \theta_{t}) \Delta t  & 0\\
    (\sin \theta_{t}) \Delta t & 0\\
    (\frac{\tan \phi_{t}}{L}) \Delta t & 0\\
    0                & 1 \Delta t
    \end{bmatrix}
    \begin{bmatrix}
    v_t+\delta_t^v \\
    \omega_t+\delta_t^\omega
    \end{bmatrix},
\end{equation}
}
where $\vect{x}_t = [p^x_t, p^y_t,\theta_t, \phi_t]^\top \in \mathbb{R}^4$ represents the horizontal coordinate, vertical coordinate, heading angle, and steering angle respectively. The input $\vect{u}_t=[v_t,\omega_t]^\top$ consists of velocity and steering angular velocity, and $\vect{\delta}_t=[\delta_t^v,\delta_t^\omega]^\top \sim \mathcal{N}(\bar{\vect{0}}, \vect{I}) $ denotes random perturbation corresponding to the input. The parameter $L=0.5$ is a wheelbase, and $\Delta t=0.1$ is the time step.

From the covariance update law~\eqref{eq:random_trajectories}, we note that $\vect{B}_t \vect{B}_t^\top$ is accumulated over time-horizon. In the above dynamics in~\eqref{eq:simple_car}, the input gain matrix can have unbounded magnitude because $\tan \phi_t$ has the singularity at $\phi_t=\pi/2$. Constraining the steering angle $\phi_t$ to avoid the singularity is a common practice for planning with the simple car kinematics. In order to qualitatively demonstrate how magnitude of the control gain would affect the performance, we consider two constraint sets for $\phi_t$: (1) wide steering angle $[-\pi/2+0.01, \pi/2-0.01]$ and (2) narrow steering angle $[-\pi/12, \pi/12]$. We generate five experiments of MPC with the finite time horizon optimal PI controls for each steering settings with $N=10000$ and $\lambda=0.1$. 
In Figure~\ref{fig:figure3_wide_steer}, the paths of the wide steering angle have greater variance than the paths with the narrow steering angle. This figure aligns with our prediction on the challenges of dealing with large input gain that causes greater variance.
\begin{figure}[thpb]
\centering
 \includegraphics[width=0.3\textwidth]{ 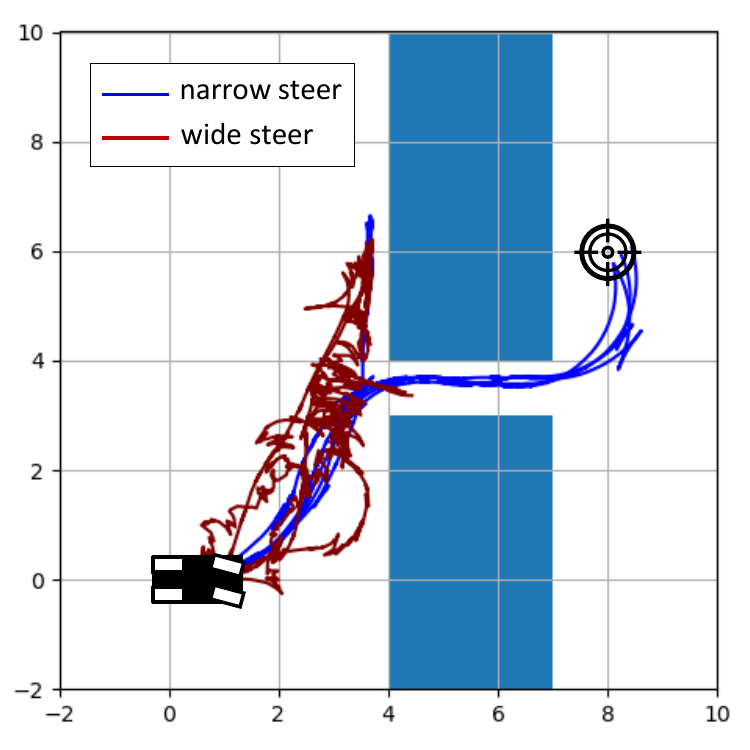}
 \caption{Sample paths of the narrow and wide steering car}
\medskip
\label{fig:figure3_wide_steer}
\vspace{-0.3in}
\end{figure}

\subsubsection{Use of Corollary~\ref{corollary:sampling_complexity_with_emphrical}}
We sample 50000 trajectories and calculate the empirical mean $\hat{E}_1 := \frac{1}{N}\sum_n^N[ \exp(-(1/\lambda)S^{(n)})]$ to apply Corollary~\ref{corollary:sampling_complexity_with_emphrical}. The result is summarized in Table~\ref{tab:sample_ugv}\footnote{For the use of Corollary~\ref{corollary:sampling_complexity_with_emphrical} in Table~\ref{tab:sample_ugv}, we used different $\lambda=10$ for the weight $w:=\exp(-(S/\lambda)$ and the target location at $(1,1)$ from Figure~\ref{fig:figure3_wide_steer} because large cost value passing to the exponential can cause numerical error, e.g., $\exp(-(40/\lambda)$.}.
\begin{table}[th]
\tiny
\caption{Required number of sample $max(N_1, N_2)$ given $\epsilon_1=0.01$, $\epsilon_2=0.1$, $\rho_1=0.05$, and $\rho_2=0.05$}
\label{tab:sample_ugv}
\begin{center}
\begin{small}
\begin{sc}
\begin{tabular}{ccrcc}
\toprule
Steering & $T$ & $N_2$    & $\frac{\epsilon_1}{\hat{E}_1}$   &$N_1$\\
\midrule
Narrow  &200   & 14517  &  0.0170   & 18444\\

\midrule
Wide   &200   & 15274  &  0.0175   & 18444\\

\bottomrule
\end{tabular}
\end{sc}
\end{small}
\end{center}
\vskip -0.2in
\vspace{-0.05in}
\end{table}


\section{CONCLUSION AND DISCUSSION}\label{sec:conclusion}
We calculated a sampling complexity bound of the Monte-Carlo integration used for the path integral methods. The sampling complexity bound depends on the expectations of the cost of the random trajectories simulated with a general nonlinear system. For useful insight, we applied the result to linear systems with quadratic state cost for target tracking and indicator penalty cost for collision avoidance. For an unstable linear system, the required number of samples increases exponentially as the time horizon increases. This result justifies the use of feedback control in the recent works that steer variances with the PI methods. Compared to the works for improved sample efficiency by heuristically modifying the controlled dynamics, our result can be used to justify such modifications in terms of sampling complexity analysis.

Future work includes studying the effect of $\lambda$ considering the trade-off between the performance and the sampling complexity. The naive idea is to increase $\lambda$ for adjustment of the variance that affects the sampling complexity. However, as the author in~\cite{kappen2005path} described, the choice of $\lambda$ does change the behavior of the resulting optimal control. Hence, depending on $\lambda$, the optimal control can behave differently from the deterministic optimal control in terms of performance.

\bibliographystyle{IEEEtran}
\bibliography{mybib}

\section*{Appendix}
 A bound on the variance of the product of $X$ and $Y$:
\begin{equation}\label{eq:bound_var_product}
    \begin{aligned}
        \Var(XY)&=\Cov(X^2,Y^2)+\E(X^2)\E(Y^2)-(\E(XY))^2 \\
                &\leq \Cov(X^2,Y^2)+\E(X^2)\E(Y^2) \\
                &\leq \left(\Var(X^2)\Var(Y^2)\right)^{1/2} + \E(X^2)\E(Y^2)
    \end{aligned}
\end{equation}
where the first equation directly follows from the definition of variance and covariance and the last inequality is due to the Cauchy–Schwarz inequality $\Cov(A,B)^2 \leq \Var(A)\Var(B)$.

\newpage

\end{document}